\documentclass[11pt]{article}
\usepackage[margin=1in]{geometry}
\usepackage[T1]{fontenc}
\usepackage[utf8]{inputenc}
\usepackage{lmodern}
\usepackage{microtype}
\usepackage{amsmath,amssymb,amsthm,mathtools}
\usepackage{bm}
\usepackage{tikz-cd}
\usepackage{multicol}
\usepackage{hyperref}
\usepackage[nameinlink,noabbrev]{cleveref}
\usepackage{booktabs}
\usepackage{longtable}
\usepackage{array}
\usepackage{enumitem}
\usepackage{verbatim}
\usepackage{xcolor}
\usepackage{titlesec}
\usepackage{titling}
\usepackage{abstract}
\usepackage{tabularx}
\usepackage{booktabs}
\newcolumntype{Y}{>{\raggedright\arraybackslash}X}
\newenvironment{paperfigure}{\begin{center}\begin{minipage}{0.94\linewidth}\small}{\end{minipage}\end{center}}

\hypersetup{colorlinks=true,linkcolor=blue!60!black,citecolor=blue!60!black,urlcolor=blue!60!black}
\setlist[itemize]{leftmargin=1.6em,itemsep=0.2em,topsep=0.35em}
\setlist[enumerate]{leftmargin=1.8em,itemsep=0.2em,topsep=0.35em}

\newtheorem{theorem}{Theorem}[section]
\newtheorem{proposition}[theorem]{Proposition}

\theoremstyle{definition}

\titleformat{\section}{\large\bfseries}{\thesection.}{0.6em}{}
\titleformat{\subsection}{\normalsize\bfseries}{\thesubsection.}{0.5em}{}
\titleformat{\subsubsection}{\normalsize\itshape}{\thesubsubsection.}{0.5em}{}
\titlespacing*{\section}{0pt}{1.2em}{0.45em}
\titlespacing*{\subsection}{0pt}{0.95em}{0.3em}
\titlespacing*{\subsubsection}{0pt}{0.75em}{0.25em}

\pretitle{\begin{center}\LARGE\bfseries}
\posttitle{\par\end{center}\vspace{-0.6em}}
\preauthor{\begin{center}\normalsize}
\postauthor{\par\end{center}\vspace{-0.8em}}
\predate{}
\postdate{}

\title{DALM: A Domain-Algebraic Language Model via Three-Phase Structured Generation}
\author{Chao Li\\\normalsize Deepleap.ai \quad \texttt{lichao@deepleap.ai}}
\date{}

\begin{document}
\maketitle

\begin{abstract}
Large language models compress human knowledge into unstructured weight vectors where domain boundaries do not exist --- a fact about quantum mechanics and a fact about cooking share the same parameter space and can contaminate each other during generation. We propose a domain-algebraic language model (DALM) that generates under exact structural constraints derived from a domain lattice.

DALM shares a core intuition with diffusion language models (dLLMs): generation as progressive denoising from high entropy to low entropy. The difference is structural. Existing dLLMs (LLaDA, Dream, Zhou et al. 2026) denoise by randomly unmasking tokens --- there is no semantic ordering, no domain constraint, no algebraic guarantee on the denoising path. DALM denoises along a domain lattice: domain uncertainty is resolved first, then relation uncertainty, then concept uncertainty. Each denoising step is algebraically constrained.

The framework requires three abstract ingredients: (1) a lattice (L, $\sqsubseteq$) of domains with computable meet, join, and implication; (2) a typing function $\tau$ that classifies relations as monotone or non-monotone, controlling inheritance across domain boundaries; and (3) a fiber function F that partitions the knowledge base into domain-local subsets. Given any system satisfying these requirements, DALM provides a three-phase encoder-decoder architecture where every generation step is confined to a domain fiber, cross-domain contamination is structurally prohibited in closed-vocabulary mode and auditably bounded in open-vocabulary mode, and a single query produces a domain-indexed multi-perspective answer space.

The architecture is trained on domain-annotated, consistency-verified structured knowledge bases where each training example carries domain annotation, relation type, and validation guarantees --- structurally richer signal than raw text. We demonstrate the framework using the CDC (Domain-Contextualized Concept Graph) knowledge representation as a concrete instantiation, and specify an evaluation protocol on medical domain crystal libraries.

\textbf{A note on reading this paper.} The core contribution is \emph{controlled structured denoising along an algebraic lattice}, not a mathematical derivation of Markov diffusion chains. DALM is not a likelihood-based generative model operating on a flat token space; it is a constrained denoising system whose computation space is shaped by domain algebra. Reading it through the lens of representation alignment or maximum-likelihood factorization will produce apparent gaps that do not exist in the framework as formulated.

\textbf{Keywords.} Domain-Algebraic Language Model, DALM, Structured Denoising, Domain Lattice, Algebraic Constraints, Diffusion Language Models, Hallucination Prevention
\end{abstract}

\section{Introduction}

\subsection{The Compression Problem}

All large language models perform the same fundamental operation: compress a corpus of human knowledge into a parameter vector $\theta \in \mathbb{R}^{d}$, then generate text by sampling from $p(x_{t+1} \mid x_{1:t}; \theta)$.

This compression is lossy in a specific and consequential way. The parameter vector $\theta$ does not preserve the domain structure of the original knowledge. A fact about quantum mechanics and a fact about cooking share the same parameter space, influence each other during training, and can contaminate each other during generation. There is no structural mechanism in $\theta$ that separates "atom is a quantum field excitation" from "atom is an indivisible particle" --- both are smeared across the same weight matrices, distinguishable only by the statistical patterns of their surrounding tokens.

This is the structural origin of hallucination. When an LLM generates text about quantum mechanics, it samples from a distribution conditioned on the entire parameter space, including regions shaped by classical physics, chemistry, cooking, and everything else in the training corpus. Cross-domain contamination is not a bug in the sampling algorithm --- it is a structural property of unstructured compression.

\subsection{Structured Compression as an Alternative}

We propose a different compression target. Instead of compressing knowledge into an unstructured vector $\theta$, compress it into a domain-indexed structured representation:

\textbf{Unstructured (LLM):} Corpus $\to$ $\theta \in \mathbb{R}^{d}$ (one vector, all domains mixed)

\textbf{Structured (DALM):} Corpus $\to$ $\{h_c(d), h_r(d), h_d\}$ for all (c, r, d) (domain-indexed embeddings)

The DALM's representation preserves domain structure: concept embeddings are indexed by domain, relation embeddings are typed by $\tau$, and domain embeddings form a lattice with algebraic operations (meet, join, implication). Generation from this representation is domain-constrained by construction --- not by post-hoc filtering but by the geometry of the representation space itself.

\subsection{Structured Denoising: What Diffusion Language Models Should Have Been}

Diffusion language models (Zhou et al., 2026; Nie et al., 2025; Ye et al., 2025) made an important architectural discovery: generation can be formulated as progressive denoising rather than left-to-right autoregressive prediction. A fully masked sequence is iteratively unmasked, with the model predicting which tokens to reveal at each step.

The limitation is that the denoising path has no semantic structure. Token i might be unmasked before token j for purely statistical reasons --- there is no principle dictating that domain-level information should be resolved before concept-level information, or that relation types should constrain what concepts can appear. A medical token and a sports token can be unmasked in the same step, in the same attention context, with no structural isolation.

\textbf{Independent physical evidence that denoising needs structure.} Sclocchi, Favero \& Wyart (2025a) proved in PNAS that diffusion models operating on hierarchically structured data exhibit a phase transition at a critical noise threshold $\varepsilon$\emph{: below $\varepsilon$}, the denoising process preserves high-level features (e.g., image class); above $\varepsilon$*, high-level features collapse to random while low-level features persist and recombine. Their follow-up work (Sclocchi et al., 2025b, ICLR 2025) further demonstrated that forward-backward diffusion experiments can probe the latent hierarchical structure of data, with correlation lengths diverging at the phase transition.

This is direct empirical evidence --- from statistical physics, not from knowledge representation --- that hierarchical structure matters for denoising. In the physics idiom of Wyart's group: above $\varepsilon$*, the hierarchical correlations decouple and the system becomes effectively one-scale --- equivalent, in our formulation, to collapsing the domain lattice to a single universal fiber. What Sclocchi et al. observe experimentally as a critical slowing-down and a divergence of correlation length is, structurally, the point at which the algebraic constraint that separates fibers can no longer be sustained by the denoising dynamics. Wyart's group observed this with physics tools on image data; we formalize it with domain algebra and prevent it architecturally by pinning the denoising schedule to the lattice rather than letting it drift freely.

DALM is the architectural response to this physical observation. Instead of allowing the hierarchy to collapse at $\varepsilon$* (which happens in unstructured diffusion), DALM enforces hierarchical structure at every denoising step through the domain lattice. The three-phase denoising path is not arbitrary --- it follows the lattice:

\textbf{Step 1 (Domain denoising):} Resolve which domain(s) the input/query belongs to. This eliminates the largest source of uncertainty: which world are we in?

\textbf{Step 2 (Relation denoising):} Within each activated domain, resolve which relations are active, constrained by $\tau$-typing. This eliminates relational uncertainty: what rules apply in this world?

\textbf{Step 3 (Concept denoising):} Within each domain-relation pair, generate specific concepts from the fiber-local vocabulary. This eliminates conceptual uncertainty: what specific entities populate this world under these rules?

Each step's output space is constrained by the previous step's result. This is algebraically guaranteed denoising --- not a learned schedule that might drift during training, but exact structural constraints derived from the domain lattice.

The existing dLLM infrastructure --- masking mechanisms, KV cache partitioning, parallel decoding --- is fully compatible with DALM. The modification is the denoising schedule: replace random unmasking with lattice-structured unmasking. This is architecturally a mask replacement, not a framework replacement.

\subsection{Contributions}

\begin{enumerate}
\item A general framework for structured denoising over algebraic lattices, applicable to any knowledge system that provides a lattice, a typing function, and a fiber partition (Section 2).
\item Three-phase encoder-decoder architecture where each phase corresponds to a level of the lattice structure (concept, relation, domain) and is constrained by domain algebra (Sections 3--4).
\item Formal analysis of hallucination as cross-domain leakage, with structural prevention in closed vocabulary and auditable bounds in open vocabulary (Section 5).
\item Multi-perspective generation: a single query produces a domain-indexed answer space (Section 5).
\item Graceful degradation: partial success yields useful components --- automatic crystallizer, domain-structured embeddings (Section 5).
\item Concrete instantiation using the CDC framework (Section 7), with evaluation protocol on medical domain knowledge bases (Section 8).
\end{enumerate}

\section{Structured Denoising over Algebraic Lattices}

This section defines the abstract algebraic requirements for DALM. Any knowledge representation system satisfying these requirements can serve as DALM's structural substrate. The definitions are self-contained; a concrete instantiation is given in Section 7.

\subsection{The Three Ingredients}

\textbf{Ingredient 1: A domain lattice (L, $\sqsubseteq$).} A partially ordered set of domains with a specialization order $\sqsubseteq$, equipped with computable meet ($\sqcap$), join ($\sqcup$), and implication ($\to$) operations. The lattice has a top element $\top$ (the universal domain) and satisfies the Heyting algebra axioms: it is distributive, bounded, and supports a pseudo-complement. Concretely: @Physics@Quantum $\sqsubseteq$ @Physics $\sqsubseteq$ $\top$, and the meet of two domains is their most specific common generalization. The implication operation d$_1$ $\to$ d$_2$ is used in the $\tau$-typing mechanism: it determines whether knowledge transfer from d$_1$ to d$_2$ is structurally licensed, and it underlies the inheritance decisions that the decoder's Phase 2 must respect. The DALM architecture uses meet and join explicitly (in domain selection and lattice loss) and implication implicitly (through the $\tau$-typing mask that governs cross-domain inheritance).

\textbf{Ingredient 2: A typing function $\tau$: R $\to$ {monotone, non-monotone}.} Each relation predicate r in the system's vocabulary is classified as either monotone (truth in a parent domain implies truth in child domains) or non-monotone (truth does not propagate downward). This classification controls inheritance: if is\_a is monotone, then "Atom is\_a Particle" in @Physics inherits into @Physics@Quantum. If contrasts\_with is non-monotone, then "Wave contrasts\_with Particle" in @Physics does \emph{not} inherit into @Physics@Quantum, because wave-particle duality dissolves this contrast at that level.

The typing function may be global ($\tau$ depends only on r) or domain-conditioned ($\tau$ depends on both r and d). The global case provides stronger algebraic guarantees; the domain-conditioned case is more expressive. Both are supported. When $\tau$ is domain-conditioned and a relation r has $\tau$(r, d$_1$) = monotone but $\tau$(r, d$_2$) = non-monotone for d$_1$ $\sqsubseteq$ d$_2$, the more restrictive classification governs: the relation is treated as non-monotone for the d$_1$ $\to$ d$_2$ inheritance path, because allowing monotone propagation into a domain that classifies it as non-monotone would violate the child domain's constraints.

\textbf{Ingredient 3: A fiber function F: L $\to$ $2^{K}$.} Each domain d $\in$ L defines a fiber F(d): the complete set of knowledge units scoped to that domain. Facts in different fibers are semantically independent: is\_a(Apple, Fruit, @Biology) and is\_a(Apple, Company, @Business) coexist without contradiction because they reside in different fibers. A query scoped to domain d is evaluated only against F(d) --- concepts in other fibers do not exist for that query.

\subsection{Knowledge Units and Validation}

The fundamental unit of structured knowledge is a tuple $\langle c, r@d, c' \rangle$ where c and c' are concepts, r is a relation predicate, and @d is a domain specification. The @d field is not metadata --- it is a structural part of the predicate arity. Any system that reads the tuple as a four-field unit automatically respects the domain scope.

A \textbf{validated knowledge unit} (which we call a \emph{crystal}) is a tuple that has passed \textbf{insertion-time validation} against its fiber F(d): the new assertion does not create cycles in acyclic relations, does not reverse established causal chains, and does not contradict existing content within the same fiber. Crystals are guaranteed to be fiber-locally consistent.

\subsection{The Structured Denoising Path}

Given the three ingredients, the structured denoising path is defined as:

\[
\text{noise} \xrightarrow{\text{Phase 1: domain}} \xrightarrow{\text{Phase 2: relation}} \xrightarrow{\text{Phase 3: concept}} \text{crystal}
\]

Each phase eliminates one dimension of uncertainty, and each phase's output space is constrained by the previous phase's result:

\begin{itemize}
\item \textbf{Phase 1 (Domain):} Select which domain(s) are relevant. Output: a probability distribution over L.
\item \textbf{Phase 2 (Relation):} Within each activated domain d, select which relations are active, subject to the $\tau$-typing constraint. Output: a set of typed relation predictions per domain.
\item \textbf{Phase 3 (Concept):} Within each domain-relation pair, generate specific concepts from the fiber-local vocabulary F(d). Output: complete tuples $\langle c, r@d, c' \rangle$.
\end{itemize}

The ordering is not arbitrary. Domain uncertainty is the highest-level uncertainty (which world are we in?); concept uncertainty is the lowest-level uncertainty (which specific entity?). Resolving them in this order ensures that each step operates in a progressively more constrained space. This mirrors the hierarchical phase transition structure observed by Sclocchi et al. (2025a): high-level features (domain) must be resolved before low-level features (concepts) for the denoising to be coherent.

\textbf{Framing: denoising dynamics, not likelihood factorization.} A useful analogy: DALM behaves like a constrained dynamical system whose denoising trajectory is shaped by domain algebra, rather than a likelihood-based generative model that must explicitly align all semantically equivalent outputs. Under the lattice + $\tau$-typing + validation constraints, illegal trajectories are structurally excluded, and legal trajectories are channeled into domain-specific regions of the output space. Semantic equivalence between expressions is expected to manifest as convergence to shared or adjacent output regions under constrained denoising --- not as collapsed loss targets in a flat probability space. We use "attractor basin" as an intuitive description of this channeling effect; a formal dynamical-systems analysis (Lyapunov stability, convergence rates) is beyond the scope of this architectural paper and is left to future work.

\section{Architecture}

\subsection{Overview}

\begin{paperfigure}
\begin{tabularx}{\linewidth}{@{}l@{\hspace{1.2em}}Y@{}}
\toprule
\textbf{Stage} & \textbf{Operation} \\
\midrule
Input & Natural-language query or partial crystal \\
Encoder: Phase 1 & Concept recognition and contextual entity embedding \\
Encoder: Phase 2 & Relation detection with $\tau$-typed constraints \\
Encoder: Phase 3 & Domain assignment in the lattice \\
Latent state & Structured crystal candidates $(c,r,d)$ in the domain-indexed bottleneck \\
Decoder: Phase 1 & Domain activation and routing \\
Decoder: Phase 2 & Relation expansion within activated fibers \\
Decoder: Phase 3 & Concept generation under fiber-local constraints \\
Output & Crystal tuples or natural-language rendering \\
\bottomrule
\end{tabularx}
\end{paperfigure}

\begin{quote}
\textbf{Figure 1} (placeholder): Overview of the DALM three-phase encoder-decoder architecture. The encoder progressively reduces entropy from token space to the crystal bottleneck (concept $\to$ relation $\to$ domain). The decoder unfolds from the domain lattice back to domain-specific outputs (domain $\to$ relation $\to$ concept). The crystal library sits at the bottleneck between encoder and decoder. Insertion-time validation gate shown between encoder output and crystal library.
\end{quote}

\begin{quote}
\textbf{Figure 2} (placeholder): Structured denoising comparison. Left: standard dLLM random unmasking (tokens revealed in arbitrary order, no semantic structure). Right: DALM lattice-structured denoising (domain resolved first, then relation, then concept; each step constrained by the previous). Phase transition threshold $\varepsilon$* (Sclocchi et al., 2025a) shown as the boundary where unstructured denoising loses hierarchical coherence.
\end{quote}

The encoder moves from concrete to abstract (concept $\to$ relation $\to$ domain). The decoder moves from abstract to concrete (domain $\to$ relation $\to$ concept). Together they form an autoencoder whose latent space is the domain lattice --- a structured, algebraically constrained space rather than an unstructured vector space.

\subsection{Encoder Phase 1: Concept Recognition}

Input: a sequence of tokens $x=(x_1,\ldots,x_N)$.

The concept layer identifies entity spans and produces \textbf{contextualized} concept embeddings:

\[
h_c^{(i)} = \text{ConceptEncoder}(x, i) \in \mathbb{R}^{d_c}
\]

This is architecturally similar to a named entity recognition layer, but producing dense embeddings rather than discrete labels. The concept encoder is a standard transformer layer operating over the full input sequence --- at this phase, no domain constraint is applied because the domain is not yet determined.

\textbf{Handling polysemy.} The concept embedding $h_c$ for "Apple" is not a context-free dictionary entry. It is a contextualized vector produced by the transformer's self-attention over the full input sentence. "Apple" in "Apple fell from the tree" and "Apple" in "Apple released a new product" produce different h\_c vectors because the surrounding tokens are different. The true domain anchoring occurs in Phase 3, but Phase 1's contextualized embeddings already carry implicit domain signals. Additionally, the cross-phase attention residuals (Section 3.10) allow Phase 3's domain assignment to feed back and refine Phase 1's concept representations --- resolving residual ambiguity that context alone did not resolve.

Output: a set of concept embeddings {h\_c\textasciicircum{}{(1)}, ..., h\_c\textasciicircum{}{(M)}} for M identified entities.

\subsection{Encoder Phase 2: Relation Typing}

Input: concept embeddings from Phase 1.

For each pair of concepts (h\_c\textasciicircum{}{(i)}, h\_c\textasciicircum{}{(j)}), the relation layer predicts:

\[
p(r | c_i, c_j) = \text{softmax}(W_r [h_c^{(i)}; h_c^{(j)}; h_c^{(i)} \odot h_c^{(j)}])
\]

where r ranges over the core predicate family (is\_a, part\_of, causes, requires, enables, contrasts\_with, analogous\_to, ...) and a null relation $\varnothing$ (no relation).

Simultaneously, the $\tau$-type is determined. Two implementation modes are available:

\textbf{Hard mode (algebraic lookup):}

\[
\tau(r) = \begin{cases} \text{monotone} & \text{if } r \in \{\text{is\_a, part\_of, requires, enables}\} \\ \text{non-monotone} & \text{if } r \in \{\text{contrasts\_with, analogous\_to}\} \end{cases}
\]

The $\tau$-type is looked up from a stored typing table. This provides exact algebraic guarantees: monotone relations inherit, non-monotone relations are blocked, with no learned drift.

\textbf{Soft mode (learned classification with algebraic regularization):}

\[
\tau_{\text{soft}}(r) = \sigma(W_\tau \cdot h_r + b_\tau) \in [0, 1]
\]

where values near 1 indicate monotone and values near 0 indicate non-monotone. A regularization term pulls the soft classification toward the hard classification:

\[
\mathcal{L}_\tau = \lambda_\tau \sum_r \| \tau_{\text{soft}}(r) - \tau_{\text{hard}}(r) \|^2
\]

Soft mode is appropriate when the monotonicity of a relation is domain-dependent or context-sensitive. Hard mode is appropriate when algebraic guarantees are paramount (medical reasoning, regulatory compliance). The choice is an engineering decision at deployment time, not an architectural constraint.

\textbf{Gradient flow through soft $\tau$ masks.} When soft $\tau$ is used to modulate cross-domain attention (Section 3.7) or generation masks, a discrete/continuous boundary arises: the mask must ultimately be applied as a multiplicative gate, but $\tau$\_soft $\in$ [0,1] is continuous. Three strategies are available. (1) \textbf{Soft masking during training, hard at inference:} during training, cross-domain attention is weighted by $\tau$\_soft(r) continuously, allowing gradient flow; at inference, a threshold ($\tau$\_soft > 0.5 $\to$ monotone) produces hard masks. (2) \textbf{Straight-through estimator:} the forward pass uses hard thresholding; the backward pass treats the threshold as identity, passing gradients through as if the mask were continuous (Bengio et al., 2013). (3) \textbf{Gumbel-softmax relaxation:} replace the hard threshold with a temperature-annealed Gumbel-softmax, starting soft and hardening during training. All three maintain end-to-end differentiability while converging to hard masks at deployment.

Output: a set of typed relation triples $\{(c_i, r_k, c_j, \tau(r_k))\}$.

\subsection{Encoder Phase 3: Domain Assignment}

Input: typed relation triples from Phase 2.

The domain layer assigns each triple to a position in the domain lattice:

\[
h_d = \text{DomainEncoder}(h_c^{(i)}, h_r, h_c^{(j)}) \in \mathbb{R}^{d_d}
\]

\[
d^* = \arg\min_{d \in D} \| h_d - e_d \|^2
\]

where $e_d \in \mathbb{R}^{d_d}$ are learned domain embeddings forming a lattice-structured embedding space. The lattice structure is enforced by a training constraint:

\[
d_1 \sqsubseteq d_2 \implies \|e_{d_1} - e_{d_2}\| \leq \|e_{d_1} - e_{d_3}\| \quad \forall d_3 \text{ s.t. } d_3 \not\sqsubseteq d_2
\]

This ensures that child domains are closer to their parents than to unrelated domains in embedding space, approximating the lattice partial order geometrically.

\textbf{Hyperbolic embedding alternative.} For lattice structures with high branching factors, hyperbolic space (e.g., the Poincar\'e ball model) provides a superior embedding geometry. Hyperbolic volume grows exponentially with radius, naturally accommodating the exponential growth of nodes at each lattice level. Domain embeddings trained in hyperbolic space (Nickel \& Kiela, 2017) can represent the full lattice hierarchy with near-zero distortion even in low dimensions (d\_d = 16--32). The domain assignment then uses hyperbolic distance: $d^* = \arg\min_d d_H(h_d, e_d)$.

\textbf{Unseen domain fallback.} During inference, the argmin operation is constrained to the finite set of known domain embeddings $e_d$ existing in the pre-trained lattice. If the encoder produces a domain embedding $h_d$ that is far from all known domains $(\min_d \|h_d-e_d\| > \text{threshold})$, two fallback strategies are available: (1) \textbf{nearest ancestor projection} --- assign to the nearest domain in the lattice, which is guaranteed to be a valid generalization of the intended domain; (2) \textbf{novel domain flag} --- mark the output as belonging to a potentially new domain not yet in the lattice, triggering a domain splitting evaluation. Strategy (1) is conservative and always valid; strategy (2) enables lattice growth but requires subsequent validation.

Output: domain-assigned tuples $\{\langle c_i, r_k@d^*, c_j \rangle\}$.

\subsection{Encoder Output: Crystallization}

The encoder's output passes through insertion-time validation before entering the crystal library:

\begin{paperfigure}
\begin{tabularx}{\linewidth}{@{}lY@{}}
\toprule
\textbf{Step} & \textbf{Validation action for candidate $\langle c,r@d,c'\rangle$} \\
\midrule
1 & Check whether insertion creates a cycle in an acyclic relation inside $F(d)$. \\
2 & Check whether insertion reverses an established causal chain inside $F(d)$. \\
3 & Check whether insertion contradicts existing content inside $F(d)$. \\
4 & If all checks pass, accept the tuple into the crystal library. \\
5 & If contradiction is detected, flag the case for domain splitting. \\
6 & If a cycle is detected, reject the tuple. \\
\bottomrule
\end{tabularx}
\end{paperfigure}

This validation gate is a symbolic operation --- it is not differentiable and does not participate in gradient computation. During training, the training data consists of pre-validated crystals, so the gate is not in the forward path. During inference (when generating new crystals), the gate sits between the encoder output and the crystal library. This is a deliberate design choice: the model learns to generate valid crystals by training on valid crystals (analogous to teacher forcing in seq2seq models), and the validation gate serves as a safety net at inference time --- rejecting invalid outputs rather than providing corrective training signal. The model may generate invalid candidates that are rejected; improving rejection rate is a training-quality problem, not a structural deficiency.

Output: validated crystals $\langle c, r@d, c' \rangle$ in the crystal library.

\subsection{Decoder Phase 1: Domain Selection}

Input: a query q (natural language or partial crystal).

The decoder first determines which domain(s) are relevant to the query:

\[
\alpha(d | q) = \text{softmax}(e_d^\top \cdot h_q / \sqrt{d_d})
\]

where $h_q$ is the query embedding and $e_d$ are the same domain embeddings used in the encoder. The decoder activates all domains with $\alpha(d \mid q)$ > $\varepsilon$ for a threshold $\varepsilon$.

\textbf{Hierarchical domain routing.} The dense softmax over all domains has complexity O(K) where K = |D|. For large domain lattices (10$^4$--10$^5$ domains), this is expensive. The domain lattice's tree-like prefix structure enables hierarchical softmax: route the query level-by-level. At each level, select among \textasciitilde{}15 children (typical branching factor). Total complexity: O(depth $\times$ branching) $\approx$ O(log K).

Output: a set of activated domains $\{d_1, \ldots, d_m\}$ with activation weights $\{\alpha_1, \ldots, \alpha_m\}$.

\subsection{Decoder Phase 2: Relation Expansion}

Input: activated domains from Phase 1.

For each activated domain $d_k$, the decoder expands valid relation types:

\[
p(r | d_k, q) = \text{softmax}(W_{rd} [e_{d_k}; h_q])
\]

This is constrained by $\tau$-typing: if the query involves a relation from a parent domain and that relation is non-monotone, it is blocked from appearing in the child domain's expansion. The domain mask M\_d enforces this:

\[
M_d(d_k, r) = \begin{cases} 1 & \text{if } \tau(r) = \text{monotone or } r \text{ is native to } d_k \\ 0 & \text{if } \tau(r) = \text{non-monotone and } r \text{ is inherited} \end{cases}
\]

\textbf{Domain-specific relation representation.} The relation embedding $h_r$ is modified by the domain context through a rank-1 update:

\[
h_{r@d} = h_r + h_d \cdot v_r^\top
\]

where v\_r is a learned relation-domain interaction vector. This allows the same relation (e.g., "causes") to have domain-specific semantics while retaining a shared base representation.

Output: a set of domain-relation pairs $\{(d_k, r_j)\}$ for each activated domain.

\subsection{Decoder Phase 3: Concept Generation}

Input: domain-relation pairs from Phase 2.

For each $(d_k,r_j)$, the decoder generates concepts from the fiber-local vocabulary:

\[
p(c | d_k, r_j, q) = \text{softmax}(W_c h_{r@d} + b_{F(d_k)})
\]

The critical constraint: the softmax ranges over \textbf{only} concepts in F(d\_k) --- the fiber-local vocabulary. Concepts outside $F(d_k)$ have no representation in this softmax and cannot receive probability mass. This is the structural source of zero cross-domain leakage (Theorem 5.1).

\textbf{Open-vocabulary fallback.} When the maximum concept probability falls below a novelty threshold $\theta_{\mathrm{novel}}$ --- $\max_c \, p(c \mid d_k, r_j, q)$ < $\theta_{\mathrm{novel}}$ --- the decoder switches to a SubwordGenerator that produces a novel concept token-by-token, conditioned on the domain-relation prefix $h_{r@d}$. Novel concepts generated this way are flagged as \textbf{provisional} --- they enter the crystal library in a tentative state and require subsequent validation or human review.

Output: complete tuples $\langle c, r@d_k, c' \rangle$ for each activated domain-relation pair.

\subsection{Output Modes}

DALM supports three output modes:

\textbf{Crystal mode:} Output the raw tuples {$\langle c, r@d, c' \rangle$} with domain labels. Suitable for knowledge base construction, machine-to-machine communication, and auditable reasoning chains.

\textbf{Multi-perspective mode:} Output all domain-specific answers simultaneously, organized by domain. Suitable for research, education, and decision support where multiple viewpoints are valuable.

\textbf{Verbalization mode:} Pass the crystal output through a lightweight language model to produce natural language text. Each domain's output is verbalized separately, maintaining domain labels.

\subsection{Cross-Phase Attention Residuals}

Inspired by the Attention Residuals mechanism (Chen et al., 2026, arXiv:2603.15031), where each layer aggregates information from all preceding layers via learned attention weights rather than fixed residual connections, we adapt this principle to DALM's three-phase architecture. Instead of a rigid pipeline, each decoder phase can attend back to all preceding phases:

\textbf{Phase 2 (Relation Expansion) receives:} \[
\text{input}_2 = \text{AttnRes}(h_{\text{phase\_1}}, h_{\text{query}})
\]

\textbf{Phase 3 (Concept Generation) receives:} \[
\text{input}_3 = \text{AttnRes}(h_{\text{phase\_1}}, h_{\text{phase\_2}}, h_{\text{query}})
\]

where AttnRes computes a softmax-weighted combination:

\[
\text{AttnRes}(h_1, ..., h_k) = \sum_{i=1}^{k} \alpha_i h_i, \quad \alpha = \text{softmax}(q_{\ell}^\top [h_1; ...; h_k] / \sqrt{d})
\]

with q\_$\ell$ a learned phase-specific query vector.

\textbf{Domain-structured attention residuals.} The cross-phase attention must respect domain boundaries. Phase 3 attending back to Phase 1's domain activations should not introduce cross-domain information that Phase 2's $\tau$-typing blocked. We enforce this by applying the same domain mask to cross-phase attention, ensuring that domain isolation is maintained even through skip connections.

\section{Training}

\subsection{Training Data}

The training corpus is a domain-annotated crystal library --- a structured knowledge base where each training example is a validated tuple $\langle c, r@d, c' \rangle$ with domain annotation (@d field), relation type and $\tau$-classification, concept identities, and a consistency guarantee (every crystal has passed insertion-time validation).

This training signal is structurally richer than text. An LLM training example is a sequence of tokens with no structural annotation. A DALM training example is a tuple with three levels of annotation (concept, relation, domain), each verified for consistency. The model does not need to discover structure from statistical patterns --- structure is given.

\subsection{Loss Function}

The total loss is a sum of three phase losses for both encoder and decoder:

\[
\mathcal{L} = \mathcal{L}_{\text{enc}} + \mathcal{L}_{\text{dec}} + \mathcal{L}_{\text{lattice}}
\]

\textbf{Encoder loss:}

\[
\mathcal{L}_{\text{enc}} = \lambda_c \mathcal{L}_{\text{concept}} + \lambda_r \mathcal{L}_{\text{relation}} + \lambda_d \mathcal{L}_{\text{domain}}
\]

where L\_concept is cross-entropy on entity span detection, L\_relation is cross-entropy on relation classification for each concept pair, and L\_domain is cross-entropy on domain assignment.

\textbf{Decoder loss:}

\[
\mathcal{L}_{\text{dec}} = \lambda_d' \mathcal{L}_{\text{domain\_sel}} + \lambda_r' \mathcal{L}_{\text{rel\_exp}} + \lambda_c' \mathcal{L}_{\text{concept\_gen}}
\]

where L\_domain\_sel is cross-entropy on domain activation, L\_rel\_exp is cross-entropy on relation prediction within the correct domain, and L\_concept\_gen is cross-entropy on concept generation within the correct fiber.

\textbf{Lattice structure loss:}

\[
\mathcal{L}_{\text{lattice}} = \sum_{d_1 \sqsubseteq d_2} \max(0, \|e_{d_1} - e_{d_2}\| - \|e_{d_1} - e_{d_3}\| + \gamma)
\]

for randomly sampled d$_3$ not on the path from d$_1$ to d$_2$. This is a margin-based triplet loss that approximates the partial order of domains in continuous embedding space --- child domains stay closer to their parents than to unrelated domains. The full algebraic operations (meet, join, implication) are performed on the discrete lattice after domain assignment, not in the embedding space itself.

\subsection{Training Procedure}

\textbf{Stage 1: Encoder pre-training.} Train the encoder on natural language $\to$ crystal pairs. Input: natural language sentences from the source corpus. Target: the corresponding crystals. This teaches the encoder to crystallize.

\textbf{Stage 2: Decoder pre-training.} Train the decoder on crystal $\to$ crystal reconstruction. Input: a crystal (possibly with masked fields). Target: the complete crystal. This teaches the decoder to generate within domain constraints.

\textbf{Stage 3: End-to-end fine-tuning.} Connect encoder and decoder through the crystal bottleneck. Fine-tune the full pipeline.

\subsection{Why Training Is Easier Than Standard LLM Training}

\begin{longtable}{>{\raggedright\arraybackslash}p{0.22\textwidth} | >{\raggedright\arraybackslash}p{0.22\textwidth} | >{\raggedright\arraybackslash}p{0.22\textwidth}}
\toprule
What the model must learn & LLM & DALM \\
\midrule
\endfirsthead
\toprule
What the model must learn & LLM & DALM \\
\midrule
\endhead
Entity boundaries & From token patterns & Supervised (crystal concepts) \\
Relation types & Implicit in weights & Supervised (crystal relations) \\
Domain scope & Not learned & Supervised (crystal @d field) \\
Consistency & Not enforced & Insertion-time validation gate \\
$\tau$-typing & Not applicable & Hard constraint from typing table \\
\bottomrule
\end{longtable}

Every level of annotation reduces the learning problem. The remaining learning problem is: given this structure, how to embed concepts, relations, and domains into a continuous space that supports generation.

\subsection{Alignment Is in the Data}

Current LLM alignment relies on RLHF and related techniques to retroactively inject structure into a model trained on unstructured data. DALM inverts this. The crystal training data is pre-aligned by construction: every crystal has passed validation (consistency), carries domain annotation (scope), and respects $\tau$-typing (inheritance rules). Factual alignment --- the constraint that generated knowledge should be consistent, domain-scoped, and inferentially sound --- is not a post-training problem in DALM. It is solved at the data level.

\section{Formal Properties}

\subsection{Hallucination as Cross-Domain Leakage}

\textbf{Definition (Domain leakage).} A generation step exhibits domain leakage when it produces a concept c' $\in$ F(d$_2$) while the active domain is d$_1$ $\neq$ d$_2$ and no bridge exists between d$_1$ and d$_2$.

\begin{theorem}[Structural leakage prevention---closed vocabulary]
In DALM's decoder, when Phase 3 operates in closed-vocabulary mode (generating only from existing concepts in $F(d_k)$), domain leakage is exactly zero.
\end{theorem}

\begin{proof}
The softmax in closed-vocabulary Phase 3 ranges over $\{h_c \mid c \in F(d_k)\}$. Concepts not in $F(d_k)$ have no representation in this softmax and cannot receive probability mass. Cross-domain concepts are structurally absent, not suppressed. This guarantee derives from the hard vocabulary restriction on the softmax range, not from the $\tau$-typing mask. Whether $\tau$ operates in hard mode or soft mode (Section 3.3) does not affect this theorem: $\tau$ governs relation inheritance in Phase 2, while the vocabulary restriction governs concept generation in Phase 3.
\end{proof}

\begin{proposition}[Bounded leakage---open vocabulary]
When Phase 3 activates the open-vocabulary fallback (Section 3.8), the generated novel concept $c_{\mathrm{new}}$ is domain-conditioned (generated with $h_{r@d}$ as prefix) but not fiber-verified. The novel concept is flagged as provisional and does not enter the crystal library without subsequent validation. Domain leakage in the open-vocabulary path is therefore non-zero but \textbf{auditable}: every provisional concept carries its generation context, and its acceptance requires passing the same validation as any other crystal.
\end{proposition}

\textbf{Corollary.} DALM's hallucination profile has two regimes. In closed-vocabulary generation, cross-domain hallucination is exactly zero by construction; errors are domain-local. In open-vocabulary generation, cross-domain contamination risk is non-zero but structurally bounded.

\textbf{Scope.} The zero-leakage guarantee is at the concept level: Phase 3 cannot produce a concept from the wrong fiber. It does not cover relation-level or domain-level errors --- Phase 2 may activate a relation that is semantically inappropriate for the selected domain, and Phase 1 may select the wrong domain entirely. These are training-quality-dependent errors, not structural violations. The architecture guarantees that \emph{given correct domain selection and relation expansion}, concept generation is domain-isolated. The earlier phases are constrained by $\tau$-typing and lattice structure but not by a hard vocabulary restriction.

\subsection{Multi-Perspective Generation}

\begin{theorem}[Completeness of multi-perspective output]
If the decoder's domain selection (Phase 1) activates all domains $d$ such that $\alpha(d \mid q) > \varepsilon$ for some threshold $\varepsilon$, then every domain-perspective on the query is represented in the output.
\end{theorem}

\begin{proof}
Each activated domain $d_k$ independently produces output through Phases 2 and 3. The outputs are concatenated with domain labels. No domain's output is suppressed by another domain's output---they are computed independently in separate fibers.
\end{proof}

\textbf{Consequence.} "What is an atom?" produces simultaneous answers from physics, chemistry, philosophy --- each domain-labeled, each internally consistent, each generated within its own fiber without cross-contamination.

\textbf{Note.} Completeness is conditional on activation quality: the theorem guarantees that activated domains produce independent, uncontaminated outputs, but whether all \emph{relevant} domains are activated depends on the learned domain embeddings and the threshold $\varepsilon$. A domain with a poorly trained embedding may fail to activate even when relevant. This is an empirical property of the trained model, not a structural guarantee.

\subsection{Graceful Degradation}

\textbf{Theorem (Component-wise utility).} Each component of DALM is independently useful:

(a) The full encoder-decoder is a domain-structured language model.

(b) The encoder (Phases 1-3) combined with the validation gate is an automatic crystallization pipeline: natural language $\to$ validated domain-scoped tuples. The encoder alone (without the gate) produces crystal candidates --- useful but not guaranteed consistent.

(c) The embedding space alone $\{h_c, h_r, e_d\}$ provides domain-structured vector representations suitable for similarity-based bridge discovery: sim(e\_{d$_1$}, e\_{d$_2$}) > $\theta$ suggests a potential cross-domain connection.

\textbf{Proof.} (a) By construction. (b) The encoder's output is a set of crystal candidates; those passing validation enter the crystal library as validated tuples. (c) The trained embeddings satisfy the lattice structure constraint (Section 4.2, L\_lattice), making them valid domain-indexed vectors. $\square$

\textbf{Consequence.} DALM has no failure mode that wastes the training investment.

\section{Encoding-Decoding Relationship}

\subsection{Encoding as Denoising}

The encoder's three phases perform progressive denoising of unstructured input:

\begin{paperfigure}
\begin{tabularx}{\linewidth}{@{}lY@{}}
\toprule
\textbf{Stage} & \textbf{Illustrative encoder state} \\
\midrule
Input & ``Atoms are indivisible particles in Dalton's model''; all concepts, relations, and domains remain mixed. \\
Phase 1 output & $[\mathrm{Atom},\ \mathrm{Indivisible\_Particle},\ \mathrm{Dalton\_Model}]$; entity uncertainty is reduced. \\
Phase 2 output & $[\mathrm{Atom}\xrightarrow{\mathrm{is\_a}}\mathrm{Indivisible\_Particle},\ \mathrm{Atom}\xrightarrow{\mathrm{context}}\mathrm{Dalton\_Model}]$; relational uncertainty is reduced. \\
Phase 3 output & $[\mathrm{is\_a}(\mathrm{Atom},\mathrm{Indivisible\_Particle},@\mathrm{Chemistry}@\mathrm{Dalton})]$; domain uncertainty is resolved. \\
\bottomrule
\end{tabularx}
\end{paperfigure}

Each phase eliminates one dimension of uncertainty. The output is a crystal --- zero uncertainty in domain assignment, zero uncertainty in relation type, near-zero uncertainty in concept identity. (We use "entropy" informally here to mean the number of remaining choices at each level, not Shannon entropy of the crystal string itself, which is nonzero.)

\subsection{Decoding as Controlled Diffusion}

The decoder's three phases perform progressive specification from abstract to concrete:

\begin{paperfigure}
\begin{tabularx}{\linewidth}{@{}lY@{}}
\toprule
\textbf{Stage} & \textbf{Illustrative decoder state} \\
\midrule
Query & ``What is an atom?'' \\
Phase 1 output & $\{@\mathrm{Physics}@\mathrm{Classical}:0.3,\ @\mathrm{Physics}@\mathrm{Quantum}:0.4,\ @\mathrm{Chemistry}@\mathrm{Dalton}:0.2,\ @\mathrm{CS}@\mathrm{Concurrent}:0.1\}$; several domains remain plausible. \\
Phase 2 output & Quantum branch: $\{\mathrm{is\_a},\mathrm{has\_attribute},\mathrm{part\_of}\}$; Dalton branch: $\{\mathrm{is\_a},\mathrm{has\_attribute}\}$; relation uncertainty decreases inside each activated fiber. \\
Phase 3 output & Quantum answer: ``field excitation; Dalton answer: ``indivisible unit; concepts become domain-specific answers. \\
\bottomrule
\end{tabularx}
\end{paperfigure}

\subsection{The Asymmetry}

The encoder is \textbf{convergent}: it takes many possible inputs and maps them to a small set of crystal structures. Many different sentences can express "atoms are particles in chemistry" --- the encoder maps all of them to the same crystal.

The decoder is \textbf{divergent}: it takes one query and maps it to many possible outputs, one per activated domain. "What is an atom?" diverges into multiple domain-specific answers.

The crystal library sits at the bottleneck --- the point of minimum entropy and maximum structure. To be precise: the crystal library is the training data source and the validation reference, not a retrieval store queried at inference time. The decoder generates from learned parameters shaped by the crystal training data, not by looking up crystals in a database. "Bottleneck" refers to the information-theoretic narrowing --- the encoder compresses free text into structured tuples, and the decoder expands structured queries into domain-specific outputs. The crystal library defines the structured representation space that both sides operate in.

\section{Instantiation: The CDC Framework}

The abstract framework of Section 2 requires three ingredients: a domain lattice, a typing function, and a fiber partition. This section describes one concrete system that provides all three: the Domain-Contextualized Concept Graph (CDC) framework (Li \& Wang, 2026a; Li, Wang \& Zhao, 2026b; Li \& Wang, 2026c).

\textbf{Domain lattice.} CDC represents domains as hierarchical strings (e.g., @Physics@Quantum, @ICD11@Respiratory@Infectious) forming a lattice under the specialization order. The lattice has a Heyting algebra structure (Li, Wang \& Zhao, 2026b, Theorem 4.9) with meet computed as the longest common prefix. The Heyting algebra is intuitionistic (not Boolean): "non-Physics" is not a well-defined domain.

\textbf{Typing function.} CDC stores the $\tau$-classification as data in a meta-fiber F(@Meta@Logic), not as hardcoded rules. The typing is formalized as a $\tau$-typed Galois connection (Li, Wang \& Zhao, 2026b, Theorem 4.19), providing exact algebraic guarantees on inheritance.

\textbf{Fiber partition.} Each CDC domain defines a fiber F(d) containing all four-tuples $\langle c, r@d, c' \rangle$ scoped to that domain. The @d field is embedded in predicate arity --- not metadata but a structural field that no arity-respecting operation can ignore (Li \& Wang, 2026a).

\textbf{Insertion-time validation.} CDC implements write-time falsification (Li \& Wang, 2026c): cycle detection in acyclic relations, causal chain consistency, and within-fiber contradiction checking. Crystals satisfy Representation-Computation Unity (RCU): writing a crystal simultaneously defines its inferential behavior; reading it constrains reasoning scope. This means DALM's output space is already a subspace that is closure-compatible --- the generator does not need to execute algebraic closure explicitly.

\textbf{Crystal libraries.} CDC crystal libraries have been validated on ICD-11 medical classification (1,247 entities across 6 domain fibers). These libraries serve as the training data for the evaluation protocol described in Section 8.

\textbf{Other possible instantiations.} The DALM framework is not limited to CDC. Any knowledge representation system providing a lattice with computable operations, a relation typing mechanism, and a fiber partition could serve as the structural substrate --- for example, a typed ontology with modular organization, a hierarchical topic model with explicit boundaries, or a domain-separated knowledge graph with inheritance rules.

\section{DALM vs. Standard LLM: Structural Comparison}

\begin{longtable}{>{\raggedright\arraybackslash}p{0.22\textwidth} | >{\raggedright\arraybackslash}p{0.22\textwidth} | >{\raggedright\arraybackslash}p{0.22\textwidth}}
\toprule
Dimension & Standard LLM & DALM \\
\midrule
\endfirsthead
\toprule
Dimension & Standard LLM & DALM \\
\midrule
\endhead
Knowledge representation & Unstructured weight vector $\theta$ & Domain-indexed structured tuples \\
Generation space & All tokens equally accessible & Fiber-local vocabulary per domain \\
Hallucination type & Any token can follow any token & Errors bounded to active fiber \\
Multi-perspective & Single answer (mode collapse) & Domain-indexed answer space \\
Auditability & Opaque (weights) & Traceable (domain-labeled tuples) \\
Training data & Raw text & Validated, domain-annotated crystals \\
Alignment & Post-hoc (RLHF) & Structural (in the data) \\
Domain isolation & None & Algebraic (Theorem 5.1) \\
Decentralization & Central parameter server & Fiber-per-node, crystal exchange \\
\bottomrule
\end{longtable}

\section{Evaluation Protocol}

This section specifies the evaluation protocol designed to validate the structural claims of DALM. As this manuscript establishes the architectural groundwork, the execution of this protocol on full crystal bases is reserved for a subsequent empirical report. The design is included here to demonstrate the falsifiability of the architectural guarantees.

\subsection{Protocol Design}

\textbf{Crystal base:} ICD-11 respiratory disease crystals (1,247 entities, \textasciitilde{}5,000 four-tuples across 6 domain fibers), instantiated via the CDC framework (Section 7).

\textbf{Encoder experiment:} Input: the natural language source sentences from which the crystals were originally extracted. Target: the crystals themselves. Measure: concept F1, relation accuracy, domain assignment accuracy.

\textbf{Decoder experiment:} Input: partial crystals (domain + relation, concept masked). Target: complete crystals. Measure: concept prediction accuracy within the correct fiber.

\textbf{End-to-end experiment:} Input: natural language medical descriptions. Output: crystal four-tuples. Measure: validation pass rate (what fraction of generated crystals are accepted without contradiction).

\textbf{Domain leakage experiment:} Deliberately prompt the decoder with cross-domain queries. Measure: fraction of generated concepts that belong to the wrong fiber (should be exactly 0 by construction).

\subsection{Baselines}

\begin{itemize}
\item \textbf{GPT-4 / Claude with structured extraction prompt}: same input, same target crystals, compare extraction quality.
\item \textbf{Standard seq2seq} (no domain structure): same training data, but without domain annotations --- flat crystal prediction.
\item \textbf{Retrieval baseline}: given a query, retrieve the nearest crystal by embedding similarity.
\end{itemize}

\subsection{Hypotheses and Expected Outcomes}

1. \textbf{Domain leakage = 0} (structural, not empirical --- follows from architecture). 2. \textbf{DALM extraction is more consistent} than prompt-based extraction (fewer validation rejections per accepted crystal). 3. \textbf{Multi-perspective output is unique to DALM} --- baselines produce one answer; DALM produces a domain-indexed answer space. 4. \textbf{DALM encoder is cheaper} than GPT-4 for crystallization (smaller model, structured output, no prompt engineering).

\subsection{Status}

Execution of this protocol requires training infrastructure for the three-phase encoder-decoder on the ICD-11 crystal base. The architectural guarantees (Hypothesis 1: zero closed-vocabulary leakage; Hypothesis 3: multi-perspective output) are structural consequences of the architecture and hold independent of empirical validation. Hypotheses 2 and 4 are empirical claims that require experimental confirmation. We report on these in a subsequent empirical companion.

\section{Discussion}

\subsection{What DALM Changes About Language Modeling}

Standard language modeling asks: "Given preceding tokens, what is the most likely next token?"

DALM asks: "Given a query, in which domains is it answerable, by which relations, producing which concepts?"

This is a shift from token-level stochastic generation to domain-structured computation. The generation process is no longer a random walk in token space but a constrained trajectory through a lattice of semantic domains.

\subsection{On Differentiability}

A common objection to combining symbolic structure with neural networks is that "symbols don't do gradients." This objection does not apply to DALM.

DALM does not place symbolic operations inside the computation graph. Every operation in the forward pass --- concept embedding, relation classification, domain assignment, domain selection, relation expansion, concept generation --- is a standard differentiable neural network operation (transformer layers, softmax, distance computations, attention). Gradients flow end-to-end through all three encoder phases and all three decoder phases.

What is symbolic is the \emph{constraint structure}, not the \emph{computation}. The domain mask M\_d, the $\tau$-typing table, the lattice structure loss --- these shape the geometry of the computation space. The analogy is to convolutional neural networks: the local connectivity constraint is "symbolic" (no learned parameter determines which pixels are neighbors), but no one claims CNNs are non-differentiable. DALM's domain structure plays the same architectural role: it defines which computations are possible, not how they are computed.

\subsection{What DALM Achieves and Does Not Achieve as Neuro-Symbolic Integration}

\textbf{What this achieves.} Symbolic structure (domain lattice, $\tau$-typing) shapes neural computation at the architectural level. The domain mask is a structural zero in the attention matrix --- not a learned approximation but an exact constraint. The $\tau$-type determines which attention connections exist. This is integration at the level of computational geometry: symbolic structure defines the manifold on which neural computation operates.

\textbf{What this does not achieve.} DALM does not perform symbolic theorem proving, logical deduction, or formal verification inside the neural network. The validation gate is genuinely symbolic and sits outside the gradient flow. If you need a formal proof that a generated crystal follows from axioms, DALM cannot provide it --- that requires a separate symbolic verifier. DALM generates structurally constrained candidates; a symbolic engine verifies them.

\textbf{Honest summary.} DALM achieves neuro-symbolic integration in the sense that symbolic structure shapes neural computation at the architectural level. It does not achieve integration in the sense of performing symbolic reasoning inside a differentiable computation graph.

\subsection{Scaling Path}

\textbf{Short-term (single domain):} Train on one domain's crystal base (\textasciitilde{}5,000 crystals). Validate architecture and measure baseline quality. This is laptop-scale.

\textbf{Medium-term (multi-domain):} Train on 10--20 domain fibers (10$^4$--10$^5$ crystals). Test cross-domain isolation, multi-perspective output, and domain selection accuracy. This requires modest GPU resources.

\textbf{Long-term (web-scale):} Train on crystal bases produced by continuous crystallization over web-scale corpora. This requires compute resources beyond the current scope.

\textbf{A natural application domain: structured code generation.} Programming languages are arguably the most mature existing instantiation of DALM's abstract requirements. Type systems are $\tau$-typing (a function declared as returning Int cannot return String --- this is a hard monotonicity constraint). Modules, packages, and namespaces are domain fibers (a function in module A cannot access private members of module B --- structural isolation, not filtering). Compilers perform insertion-time validation (type checking, scope resolution, cycle detection in dependencies). The correspondence is not analogical but structural: a well-typed program is a crystal whose consistency has been verified by the compiler against its module's fiber. DALM trained on typed, module-scoped code repositories would generate code that is domain-isolated by construction, with the compiler serving as the validation gate --- a path toward near-zero-error structured code generation with built-in formal verification.

\subsection{Limitations}

\textbf{Cold start.} DALM requires a crystal library for training. Crystal libraries require a knowledge extraction process --- which may itself use an LLM. The bootstrap dependency is: LLM $\to$ crystals $\to$ DALM. Once DALM is trained, its encoder can serve as the extraction front end, breaking the dependency --- but the first generation requires an existing LLM.

\textbf{Domain lattice coverage.} DALM can only generate within domains present in its training crystal base. New domains require either crystal library expansion or domain-interpolation in embedding space (speculative, unvalidated).

\textbf{Natural language fluency.} DALM's output in crystal mode is structured but not fluent. The verbalization layer (Section 3.9) converts crystals to natural language, but the quality depends on a separate generation step.

\textbf{Concept vocabulary.} Phase 3's closed-vocabulary path generates from a fixed concept set per fiber. The open-vocabulary fallback (Section 3.8) provides domain-conditioned novel concept generation but with provisional status. Fully open-ended concept generation with algebraic guarantees remains an open problem.

\section{Related Work}

\subsection{Diffusion Language Models}

LLaDA (Nie et al., 2025), Dream (Ye et al., 2025), and the dLLM framework (Zhou et al., 2026) established discrete diffusion for text generation. Fast-dLLM (Wu et al., 2026, ICLR 2026) showed that KV caching and parallel decoding are compatible with diffusion LMs. These works provide the engineering substrate --- masking, caching, parallel decoding --- that DALM extends with structured denoising. The contribution of DALM relative to dLLMs is the denoising schedule: replace random unmasking with lattice-structured resolution.

\subsection{Hierarchical Structure and Phase Transitions in Deep Learning}

Sclocchi, Favero \& Wyart (2025a) proved that diffusion models exhibit phase transitions when operating on hierarchically structured data. Cagnetta et al. (2024) analyzed how deep networks learn compositional data via the Random Hierarchy Model. These results provide the physical foundation for DALM's design: hierarchical structure in data necessitates structured denoising for coherent generation.

\subsection{Structured Generation}

Constrained decoding (Hokamp \& Liu, 2017) restricts LLM output to satisfy lexical constraints. Grammar-constrained generation (Shin et al., 2021) restricts output to valid programs. DALM's domain constraint is more fundamental: it restricts the generation space to a semantically defined fiber, not to a syntactic pattern.

\subsection{Knowledge-Grounded Generation}

REALM (Guu et al., 2020), RAG (Lewis et al., 2020), and Atlas (Izacard et al., 2022) retrieve knowledge to augment LLM generation. Retrieved knowledge is unstructured text chunks. DALM's knowledge is structured crystals with domain annotations and algebraic guarantees.

\subsection{Neuro-Symbolic Generation}

DeepProbLog, NeuralLog, and related systems combine neural networks with symbolic reasoning. The neural and symbolic components are separate, connected by a bridge. In DALM, the three-phase architecture embeds symbolic structure directly into the neural generation process --- not as an external verifier but as architectural constraints on the generation space.

\section{Conclusion}

Large language models compress knowledge into unstructured weights. DALM compresses knowledge into domain-indexed structured representations and generates from those representations under algebraic constraints.

The contribution is a general framework for structured denoising over algebraic lattices. Given any system providing a domain lattice, a typing function, and a fiber partition, DALM constructs a three-phase encoder-decoder where every generation step is domain-constrained by construction. Cross-domain contamination --- the structural source of hallucination --- is architecturally prevented in closed vocabulary and auditably bounded in open vocabulary.

The framework is concrete: we provide the full architecture, loss functions, training procedure, formal properties (zero-leakage theorem, multi-perspective completeness, graceful degradation), and an evaluation protocol. The CDC knowledge representation framework serves as a concrete instantiation, with validated crystal libraries available for training.

Diffusion language models discovered that generation can be denoising. DALM adds the missing insight: denoising should have structure. Not random unmasking over a flat token space, but algebraically constrained resolution along a domain lattice --- domain first, then relation, then concept.

A DALM is inherently decentralizable --- domain-structured knowledge is distributable by construction. Each fiber can reside on a separate node. Crystal exchange between nodes is a protocol-level operation. No central parameter server is needed. This is what language modeling looks like when the knowledge infrastructure beneath it is crystallized rather than melted.

\section{Acknowledgments}

The theoretical foundation of DALM's structured denoising path owes a decisive debt to the work of Matthieu Wyart's group at EPFL. Sclocchi, Favero \& Wyart's demonstration that diffusion models exhibit phase transitions corresponding to hierarchical data structure (PNAS 2025) provided the independent physical evidence that hierarchical structure is not optional but fundamental to the denoising process. Their subsequent work on probing latent hierarchical structure via diffusion models (ICLR 2025), together with Cagnetta et al.'s analysis of how deep networks learn compositional data through the Random Hierarchy Model (Physical Review X, 2024), established the empirical and theoretical ground on which DALM's algebraic formalization of structured denoising stands. The correspondence is concrete: the critical noise threshold $\varepsilon$* at which high-level features decouple in their experiments is, in our language, the point at which the fiber bundle degenerates to a trivial product; the correlation-length divergence they measure is the statistical-physics signature of the algebraic constraint losing its hold on the denoising trajectory. Without their results, the claim that "denoising should have structure" would be an architectural preference rather than a physically grounded necessity.

We also acknowledge the dLLM community --- particularly Zhou et al. (2026), the LLaDA and Dream teams, and the Fast-dLLM developers (Wu et al., 2026, ICLR 2026) --- for building the open-source diffusion language model infrastructure that DALM is designed to extend. The Kimi team at Moonshot AI (Chen et al., 2026) demonstrated with Attention Residuals that depth-wise information aggregation can be fundamentally improved through learned attention over preceding layers, directly informing DALM's cross-phase attention residual design.

\normalsize

\section{References}
\begin{thebibliography}{99}
\small
\bibitem{bengio2013} Bengio, Y., L\'eonard, N., \& Courville, A. (2013). Estimating or propagating gradients through stochastic neurons for conditional computation. \emph{arXiv:1308.3432}.

\bibitem{cagnetta2024} Cagnetta, F., Petrini, L., Tomasini, U. M., Favero, A., \& Wyart, M. (2024). How deep neural networks learn compositional data: The random hierarchy model. \emph{Physical Review X}, 14, 031001.

\bibitem{chen2026} Chen, G., Zhang, Y., Su, J., et al. (2026). Attention Residuals. Kimi Team, Moonshot AI. \emph{arXiv:2603.15031}.

\bibitem{guu2020} Guu, K., Lee, K., Tung, Z., Pasupat, P., \& Chang, M.-W. (2020). REALM: Retrieval-augmented language model pre-training. \emph{ICML 2020}.

\bibitem{hokamp2017} Hokamp, C., \& Liu, Q. (2017). Lexically constrained decoding for sequence generation using grid beam search. \emph{ACL 2017}.

\bibitem{izacard2022} Izacard, G., Lewis, P., Lomeli, M., Hosseini, L., Petroni, F., Schick, T., Dwivedi-Yu, J., Joulin, A., Riedel, S., \& Grave, E. (2022). Atlas: Few-shot Learning with Retrieval Augmented Language Models. \emph{arXiv:2208.03299} [cs.CL].

\bibitem{lewis2020} Lewis, P., et al. (2020). Retrieval-augmented generation for knowledge-intensive NLP tasks. \emph{NeurIPS 2020}.

\bibitem{li2026a} Li, C., Wang, Y., \& Zhao, C. (2026a). Domain-constrained knowledge representation: A modal framework. \emph{arXiv:2604.01770} [cs.AI].

\bibitem{li2026b} Li, C., Wang, Y., \& Zhao, C. (2026b). Domain-Contextualized Inference: A Computable Graph Architecture for Explicit-Domain Reasoning. \emph{arXiv:2604.04344} [cs.AI].

\bibitem{li2026c} Li, C., Wang, Y., \& Zhao, C. (2026c). Reasoning as Data: Representation-Computation Unity and Its Implementation in a Domain-Algebraic Inference Engine. \emph{arXiv:2604.10908} [cs.AI].

\bibitem{nickel2017} Nickel, M., \& Kiela, D. (2017). Poincar\'e embeddings for learning hierarchical representations. \emph{NeurIPS 2017}.

\bibitem{nie2025} Nie, S., Zhu, F., You, Z., Zhang, X., Ou, J., Hu, J., Zhou, J., Lin, Y., Wen, J.-R., \& Li, C. (2025). Large Language Diffusion Models. \emph{arXiv:2502.09992} [cs.CL].

\bibitem{sclocchi2025a} Sclocchi, A., Favero, A., \& Wyart, M. (2025a). A phase transition in diffusion models reveals the hierarchical nature of data. \emph{Proceedings of the National Academy of Sciences}, 122(1), e2408799121.

\bibitem{sclocchi2025b} Sclocchi, A., Favero, A., Levi, N. I., \& Wyart, M. (2025b). Probing the latent hierarchical structure of data via diffusion models. \emph{ICLR 2025}.

\bibitem{shin2021} Shin, R., Lin, C. H., Thomson, S., Chen, C., Roy, S., Platanios, E. A., ... \& Klein, D. (2021). Constrained language models yield few-shot semantic parsers. \emph{EMNLP 2021}.

\bibitem{wu2026} Wu, C., Zhang, H., Xue, S., Liu, Z., Diao, S., Zhu, L., Luo, P., Han, S., \& Xie, E. (2026). Fast-dLLM: Training-free Acceleration of Diffusion LLM by Enabling KV Cache and Parallel Decoding. \emph{ICLR 2026 Poster}.

\bibitem{ye2025} Ye, J., et al. (2025). Dream: Discrete denoising diffusion for text generation. \emph{arXiv:2503.03831}.

\bibitem{zhou2026} Zhou, Z., Chen, L., Tong, H., \& Song, D. (2026). dLLM: Simple diffusion language modeling. \emph{arXiv:2602.22661}.

\end{thebibliography}
\end{document}